\documentclass[11pt]{article}

\usepackage[utf8]{inputenc}
\usepackage[T1]{fontenc}
\usepackage{lmodern}
\usepackage[protrusion=true,expansion=false]{microtype}
\usepackage[margin=1in, top=1.1in, bottom=1.1in]{geometry}

\usepackage{amsmath,amssymb,amsthm}
\usepackage{booktabs}
\usepackage{graphicx}
\usepackage[colorlinks=true,
            linkcolor=blue!70!black,
            citecolor=blue!70!black,
            urlcolor=blue!70!black]{hyperref}
\usepackage{xcolor}
\usepackage{enumitem}
\usepackage[numbers,sort&compress]{natbib}
\usepackage{caption}
\usepackage{tikz}
\usetikzlibrary{shapes.geometric,arrows.meta,positioning,fit,backgrounds,patterns}

\usepackage{titlesec}
\titleformat{\section}{\large\bfseries}{\thesection}{1em}{}
\titleformat{\subsection}{\normalsize\bfseries}{\thesubsection}{1em}{}

\newtheorem{theorem}{Theorem}[section]
\newtheorem{lemma}[theorem]{Lemma}
\newtheorem{corollary}[theorem]{Corollary}

\newtheorem{definition}{Definition}[section]
\newtheorem{assumption}{Assumption}[section]
\newtheorem{example}{Example}[section]
\newtheorem{remark}{Remark}[section]

\newcommand{\Sigstar}{\Sigma^{*}}
\newcommand{\Azero}{\mathcal{A}_0}
\newcommand{\Dhat}{\widehat{D}}

\newcommand{\sigg}{\sigma(g)}

\newcommand{\Ind}[1]{\mathbf{1}_{[#1]}}
\newcommand{\JS}{\mathrm{JS}}
\newcommand{\E}{\mathbb{E}}
\newcommand{\Prob}{\mathbb{P}}
\newcommand{\Compliance}{\mathrm{Compliance}}

\title{\textbf{From Admission to Invariants:}\\
       \large Measuring Deviation in Delegated Agent Systems}

\author{Marcelo Fernandez \\ \small TraslaIA \\ \small \texttt{info@traslaia.com}}
\date{}

\begin{document}
\maketitle

\begin{abstract}
Autonomous agent systems are governed by enforcement mechanisms that flag
hard constraint violations at runtime.  The Agent Control Protocol
\citep{fernandez2026acp} identifies a structural limit of such systems
(\S17): a correctly-functioning enforcement engine can enter a regime in
which behavioral drift is invisible to it, because the enforcement signal
operates below the layer where deviation is measurable.
We show that enforcement-based governance is \emph{structurally} unable to
determine whether an agent's behavior
remains within the admissible behavior space $\Azero$ established at admission
time.  Our central result, the \textbf{Non-Identifiability Theorem}, proves
that $\Azero \notin \sigg$: the $\sigma$-algebra generated by the enforcement
signal~$g$ does not contain~$\Azero$ under the \emph{Local Observability
Assumption}, which every practical enforcement system satisfies.
The impossibility arises from a fundamental mismatch: $g$ evaluates actions
locally against a point-wise rule set, while $\Azero$ encodes global,
trajectory-level behavioral properties set at admission time.
We then define the \textbf{Invariant Measurement Layer} (IML), which bypasses
this limitation by retaining direct access to the generative model of
$\Azero$.  We prove an information-theoretic impossibility for enforcement-based
monitoring; separately, we show that IML detects admission-time drift with
provably finite detection delay, operating in the region where enforcement is
structurally blind.
We validate the theoretical claims across four experimental settings:
three controlled drift scenarios (300 and 1000 steps), a live \textbf{n8n}
webhook pipeline, and a \textbf{LangGraph} \texttt{StateGraph} agent with
deterministic tool selection---in every case, enforcement triggers zero
violations while IML deviation grows monotonically, detecting each drift
type within 9--258 steps of drift onset.
This paper is Paper~2 of a 6-paper Agent Governance Series:
atomic decision boundaries (P0,~\cite{fernandez2026a}),
stateful enforcement---ACP (P1,~\cite{fernandez2026acp}),
fair multi-agent allocation (P3,~\cite{fairgov26}),
composition irreducibility (P4,~\cite{fernandez2026comp}),
and runtime execution validity---RAM (P5,~\cite{fernandez2026ram}).
\end{abstract}

\tableofcontents
\newpage

\section{Introduction}
\label{sec:intro}

Multi-agent systems delegate subtasks to agents that operate with partial
autonomy under a shared policy framework.  At \emph{admission time}, when an
agent is authorized and its operational scope defined, the system implicitly
records the constraints, initial context, and delegation lineage that
constitute the agent's \emph{admissible behavior space}~$\Azero$.  The
dominant governance model then attempts to maintain alignment through an
\emph{enforcement signal}~$g : \Sigstar \to \{0,1\}$, which returns~$1$
whenever a trace~$\tau$ contains a hard constraint violation.

This architecture has a fundamental observability limitation that has received
little formal attention.  Enforcement signals are \emph{local}: they evaluate
each action against a rule set that is static, does not change with trajectory
history, and does not reference~$\Azero$ directly.  An agent can therefore
\emph{drift}---systematically shifting its behavioral distribution away from
admission-time expectations---while every individual action remains within the
permitted action space.  This phenomenon is not a corner case; it is the
generic behavior of any gradual distributional shift, goal reinterpretation,
or delegation-depth creep that stays below hard constraint thresholds.

\paragraph{Motivation.}
A foundational result in agent governance \citep{fernandez2026a} establishes
that only \emph{atomic decision systems}---where evaluation and execution
occur at the same state transition---can guarantee admissibility at runtime.  Split systems that separate evaluation from
execution cannot close this gap regardless of policy sophistication.
\citet{fernandez2026acp} instantiates this principle as the Agent Control
Protocol (ACP): an atomic admission-control mechanism that requires
execution-trace state.  Section~17 of that work identifies a structurally
distinct failure mode that persists \emph{even within} an atomic enforcement
layer: a correctly-functioning, stateful enforcement system can enter a regime
in which behavioral drift is invisible to it.
ACP calls this \emph{deviation collapse} (\S17.4)---enforcement is active and
syntactically correct, yet the behavioral boundary is never exercised because
upstream conditions suppress all activating inputs.

This paper operates above the atomic enforcement boundary and formalizes
the conditions under which behavioral drift escapes detection.
We prove (Theorem~\ref{thm:nonid}) that the admission-time behavioral
contract $\Azero$ is \emph{not identifiable} from the enforcement signal
$g$ alone---no matter how sophisticated the risk-scoring function.
We then introduce the \textbf{Invariant Measurement Layer} (IML), an
estimator that restores observability at the behavioral level, operating
precisely in the region where $g(\tau) = 0$ always and enforcement-based
monitoring is structurally blind.

\paragraph{Main contributions.}
\begin{enumerate}[leftmargin=1.5em, itemsep=1pt]
  \item \textbf{T1 (Existence).}  We prove that the compliance-invariance
        gap is non-empty: trajectories satisfying $g(\tau)=0$ yet
        $\tau \notin \Azero$ exist under mild conditions
        (Lemma~\ref{lem:existence}).

  \item \textbf{T2 (Non-Identifiability).}  We prove $\Azero \notin \sigg$
        under the Local Observability Assumption: no measurable function of
        the enforcement signal can reconstruct $\Azero$-membership
        (Theorem~\ref{thm:nonid}).  We exhibit an explicit constructive
        witness grounded in experimental data (Example~\ref{ex:t2}).

  \item \textbf{T3 (IML Recoverability).}  We define IML and prove it is a
        consistent estimator of deviation $D(\tau,\Azero)$ with provably
        finite detection delay (Theorem~\ref{thm:iml}).

  \item \textbf{Empirical validation.}  Four experimental settings confirm
        $g(\tau_t)=0\;\forall t$ while $\Dhat_t$ grows monotonically:
        three controlled drift scenarios (300 and 1000 steps), a live
        \textbf{n8n} webhook pipeline (\S\ref{sec:real_traces}), and a
        real \textbf{LangGraph} \texttt{StateGraph} agent
        (\S\ref{sec:langgraph})---directly instantiating both T2 and T3.
\end{enumerate}

\paragraph{Paper organization.}
Section~\ref{sec:setup} introduces the formal model.
Section~\ref{sec:theory} states and proves T1, T2, and T3.
Section~\ref{sec:iml} details the IML estimator.
Section~\ref{sec:experiments} presents the empirical validation.
Section~\ref{sec:related} discusses related work, and
Section~\ref{sec:conclusion} concludes.

\section{Problem Setup}
\label{sec:setup}

\subsection{Trace language and admissible behavior}

Let $\Sigma$ be a finite alphabet of agent actions (tool calls, delegation
requests, context updates).  A \emph{trace} is a finite sequence
$\tau = (b_1, b_2, \ldots, b_T) \in \Sigstar$.

\begin{definition}[Admissible Behavior Space]
\label{def:A0}
The \emph{admissible behavior space} is $\Azero = f(C, E_0, L) \subseteq \Sigstar$,
where $C$ is the constraint set at admission, $E_0$ is the initial context, and
$L$ is the delegation lineage.
\end{definition}

$\Azero$ is a \emph{global} object: whether $\tau \in \Azero$ depends on
full-trajectory properties (distribution over actions, depth profile, context
coherence), not on any single action in isolation.

\subsection{Enforcement signal and local observability}

\begin{definition}[Enforcement Signal]
\label{def:enforcement}
An \emph{enforcement signal} $g : \Sigstar \to \{0,1\}$ returns $1$ to
indicate a policy violation.
\end{definition}

\begin{assumption}[Local Observability]
\label{asm:local}
The enforcement signal decomposes as $g(\tau) = h(V(\tau))$,
where $V : \Sigstar \to \mathcal{V}$ evaluates only
\emph{point-wise constraint violations}
(e.g., presence of a forbidden tool, delegation depth exceeding a fixed limit)
and $h : \mathcal{V} \to \{0,1\}$ aggregates violations.
Crucially, $V(\tau)$ is independent of~$\Azero$ and depends only on
individual actions, not on global trajectory properties.
\end{assumption}

Assumption~\ref{asm:local} is satisfied by every practical enforcement
mechanism we are aware of: permission-check middleware, guardrail classifiers,
schema validators, and policy engines all evaluate individual actions against
a static rule set.

\begin{remark}[ACP as a natural instance of Definition~\ref{def:enforcement}]
\label{rem:acp-instance}
The Agent Control Protocol \citep{fernandez2026acp} provides a canonical
instance of Definition~\ref{def:enforcement} that satisfies
Assumption~\ref{asm:local}.  For a trace $\tau = (b_1,\ldots,b_T)$ of
agent requests, define the ACP projection:
\[
  g_{\mathrm{ACP}}(\tau)
  \;:=\;
  \mathbf{1}\!\left[\,\exists\,i \leq T :
    \mathrm{decision}(b_i) \in \{\texttt{Denied},\,\texttt{Escalated}\}
  \right].
\]
Each $\mathrm{decision}(b_i)$ is computed by ACP's
\texttt{evaluate-then-mutate} pipeline against a static
\texttt{PatternKey} rule set, independently of $\Azero$; it therefore
satisfies the decomposition $g = h \circ V$ required by
Assumption~\ref{asm:local}.
We use $g_{\mathrm{ACP}}$ as the running example throughout this paper.
Importantly, the Non-Identifiability Theorem (Theorem~\ref{thm:nonid}) is
proved for the \emph{entire class} of $g$ satisfying
Assumption~\ref{asm:local}, not for $g_{\mathrm{ACP}}$ alone; $g_{\mathrm{ACP}}$
is a natural \emph{instance}, not the sole scope of the result.
\end{remark}

\subsection{Observability structure}

The enforcement signal $g$ partitions $\Sigstar$ into two equivalence classes:
$g^{-1}(0)$ (compliant traces) and $g^{-1}(1)$ (violating traces).  The
corresponding $\sigma$-algebra is
\[
  \sigg \;=\; \bigl\{\,\emptyset,\; g^{-1}(0),\; g^{-1}(1),\; \Sigstar\,\bigr\}.
\]
A set $S \subseteq \Sigstar$ is \emph{identifiable from~$g$} if and only if
$S \in \sigg$.

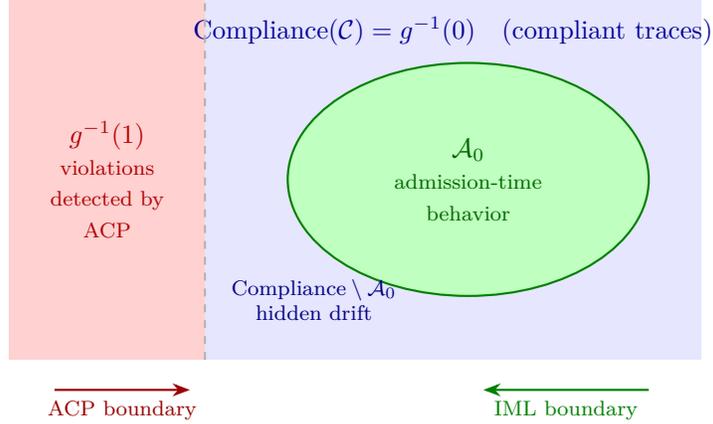
\begin{figure}[t]
\centering
\begin{tikzpicture}[
  every node/.style={font=\small},
  >=Stealth
]
\fill[gray!12] (-4.6,-2.4) rectangle (4.6,2.4);
\node[anchor=north west, gray!70!black] at (-4.55,2.35)
  {$\Sigstar$ (all traces)};

\fill[red!18] (-4.6,-2.4) rectangle (-2.0,2.4);
\node[text=red!70!black, align=center] at (-3.3,0)
  {$g^{-1}(1)$\\{\scriptsize violations}\\{\scriptsize detected by}\\{\scriptsize ACP}};

\fill[blue!10] (-2.0,-2.4) rectangle (4.6,2.4);
\node[anchor=north, blue!60!black] at (1.3,2.3)
  {$\Compliance(\mathcal{C}) = g^{-1}(0)$ \ \ (compliant traces)};

\fill[green!25!white] (1.5,0) ellipse (2.4cm and 1.55cm);
\draw[green!50!black, thick] (1.5,0) ellipse (2.4cm and 1.55cm);
\node[green!40!black, align=center] at (1.5,0)
  {$\Azero$\\{\scriptsize admission-time}\\{\scriptsize behavior}};

\node[blue!50!black, align=center, font=\scriptsize] at (-0.55,-1.6)
  {$\Compliance \setminus \Azero$\\hidden drift};

\draw[dashed, gray!60, thick] (-2.0,-2.4) -- (-2.0,2.4);

\draw[->, thick, red!60!black] (-4.0,-2.8) -- (-2.2,-2.8)
  node[midway, below, font=\scriptsize] {ACP boundary};
\draw[->, thick, green!50!black] (3.9,-2.8) -- (1.7,-2.8)
  node[midway, below, font=\scriptsize] {IML boundary};

\end{tikzpicture}
\caption{Layered structure of the trace space $\Sigstar$.
  ACP's enforcement signal $g$ partitions $\Sigstar$ into
  $g^{-1}(1)$ (constraint violations, detected by ACP) and
  $\Compliance(\mathcal{C}) = g^{-1}(0)$ (compliant traces, invisible to
  enforcement).  Within $\Compliance(\mathcal{C})$, IML monitors the
  boundary between $\Azero$ (admission-time behavior) and
  $\Compliance \setminus \Azero$ (compliant but drifted behavior---
  \emph{hidden drift}).  The two mechanisms cover complementary boundaries;
  neither is redundant.}
\label{fig:layered}
\end{figure}

\subsection{Ground-truth deviation}

\begin{definition}[Deviation Function]
\label{def:deviation}
Let $d : \Sigstar \times \Sigstar \to \mathbb{R}_{\geq 0}$ be a trajectory
metric.  The \emph{deviation} of~$\tau$ from~$\Azero$ is
$D(\tau, \Azero) = \inf_{x \in \Azero} d(\tau, x)$.
\end{definition}

\begin{remark}[Series notation]
\label{rem:notation-D}
Throughout this paper, $D(\cdot,\cdot)$ denotes a \emph{deviation function}
(real-valued distance from a behavioral set).  In Paper~0
\citep{fernandez2026a}, the symbol $D$ denotes the \emph{decision domain}
$D = \{\textsc{Allow}, \textsc{Refuse}, \textsc{Escalate}\}$---a finite set
of admission outcomes.  The two uses are unrelated in type (function vs.\
set); when results from both papers appear together, we write
$D(\tau,\Azero)$ for the deviation function and $\mathcal{D}$ or
$\{\textsc{Allow},\ldots\}$ for the decision domain to avoid ambiguity.
\end{remark}

$D(\tau,\Azero)$ is the ground truth but is intractable to compute directly
because $\Azero$ is implicitly defined.  Section~\ref{sec:iml} constructs IML
as a computable, consistent approximation.

\section{Theoretical Results}
\label{sec:theory}

\subsection{T1: Existence of the compliance-invariance gap}

\begin{lemma}[Existence]
\label{lem:existence}
Under Assumption~\ref{asm:local}, if $\Azero \neq \emptyset$ and
$\Azero \subsetneq \Sigstar$, then there exists $\tau \in \Sigstar$ such
that $V(\tau) = \emptyset$, $g(\tau) = 0$, and $\tau \notin \Azero$.
\end{lemma}

\begin{proof}
Since $\Azero \subsetneq \Sigstar$, there exists $\tau^* \notin \Azero$.
If $V(\tau^*) = \emptyset$ already, we are done.  Otherwise, $\tau^*$ contains
a step $b_i$ that triggers $V \neq \emptyset$.  By Assumption~\ref{asm:local},
$V$ depends only on point-wise properties of individual actions.  Let $\tau'$
be obtained from $\tau^*$ by replacing every $b_i$ with $V(b_i)\neq\emptyset$
with an action $b'_i$ drawn from the marginal distribution of $\Azero$
restricted to safe actions (non-empty by $\Azero \neq \emptyset$).  This
substitution satisfies $V(\tau') = \emptyset$ by construction.

It remains to show $\tau' \notin \Azero$.  Recall that $\tau^* \notin \Azero$
because it violates a trajectory-level constraint encoded in $f(C, E_0, L)$---
for instance, the fraction of safe tools in $\tau^*$ falls below the threshold
specified in $C$, or the mean delegation depth exceeds the lineage bound in $L$.
The substitution replaces only actions with $V(b_i) \neq \emptyset$ by safe
actions, which either leaves the violating distributional property intact
(if $\tau^*$ has too many boundary tools) or reduces the safe-tool fraction
further (if the original violation was that boundary tools were over-represented).
In either case, the distributional constraint that caused $\tau^* \notin \Azero$
is still violated by $\tau'$, so $\tau' \notin \Azero$, as required.

In our empirical instantiation (Section~\ref{sec:experiments}), this
construction is realized concretely: all 900 post-drift steps across three
scenarios satisfy $V(\tau_t) = \emptyset$ while $D(\tau_t, \Azero) > 0$.
\end{proof}

\subsection{T2: Non-identifiability of $\Azero$ under enforcement}

\begin{theorem}[Non-Identifiability]
\label{thm:nonid}
Under Assumption~\ref{asm:local}, if $\Azero \neq \emptyset$ and
$\Azero \subsetneq \Sigstar$, then
\[
  \Azero \notin \sigg.
\]
Equivalently, there exists no measurable function
$f : \{0,1\} \to \{0,1\}$ such that
$f(g(\tau)) = \Ind{\tau \in \Azero}$ for all $\tau \in \Sigstar$.
\end{theorem}

\begin{proof}
By Lemma~\ref{lem:existence}, there exists $\tau_2 \notin \Azero$ with
$g(\tau_2) = 0$.  Since $\Azero \neq \emptyset$, there exists
$\tau_1 \in \Azero$.  If $g(\tau_1) = 0$, we already have
$g(\tau_1) = g(\tau_2) = 0$ with $\tau_1 \in \Azero$ and $\tau_2 \notin \Azero$.
If $g(\tau_1) = 1$, then by the Local Observability Assumption~\ref{asm:local},
$V(\tau_1) \neq \emptyset$; applying the same substitution argument as in
Lemma~\ref{lem:existence} yields $\tau_1' \in \Azero$ with
$V(\tau_1') = \emptyset$, hence $g(\tau_1') = 0$.

In either case, we have $\tau_1 \in \Azero$ and $\tau_2 \notin \Azero$ with
$g(\tau_1) = g(\tau_2) = 0$.  Suppose for contradiction that
$\Azero \in \sigg$.  Since $\sigg$ is generated by $g^{-1}(0)$ and
$g^{-1}(1)$, and $\Azero \subseteq g^{-1}(0)$, any $\sigg$-measurable
function $f$ satisfies $f(0) = c$ for some constant $c \in \{0,1\}$.  But
then $f(g(\tau_1)) = f(0) = c = f(g(\tau_2))$, so
$\Ind{\tau_1 \in \Azero} = \Ind{\tau_2 \in \Azero}$.
This contradicts $\tau_1 \in \Azero$ and $\tau_2 \notin \Azero$.  \qed
\end{proof}

\begin{corollary}[Information-Theoretic Bound]
\label{cor:info}
$I(\Azero \,;\, g(\tau)) < H(\Azero)$.
The mutual information between the enforcement signal and $\Azero$-membership
is strictly less than the entropy of $\Azero$-membership.
\end{corollary}

\begin{proof}
Suppose for contradiction that $I(\Azero;\,g(\tau)) = H(\Azero)$.
This equality holds if and only if $\Azero$-membership is a deterministic
function of $g(\tau)$---i.e., there exists $f:\{0,1\}\to\{0,1\}$ with
$f(g(\tau)) = \mathbf{1}_{[\tau\in\Azero]}$ for all $\tau$.
But Theorem~\ref{thm:nonid} proves exactly that no such $f$ exists.
Hence $I(\Azero;\,g(\tau)) < H(\Azero)$.
\end{proof}

\begin{remark}[Why T2 is non-trivial]
\label{rem:nontrivial}
A referee might object: ``any many-to-one function loses information.''
T2 is stronger for three reasons.

\emph{(i) Structural, not incidental.}
The result holds for the \emph{entire class} of $g$ satisfying
Assumption~\ref{asm:local}---not for one particular rule set.
No redesign of enforcement rules, no matter how sophisticated, can make
$\Azero \in \sigg$, because the impossibility is a consequence of the
local-observation architecture itself.

\emph{(ii) The oracle contrast.}
If $g$ were permitted to query $\Azero^{\mathrm{emp}}$ at each step---i.e.,
to evaluate the JS divergence $\JS(P_\tau \| P_{E_0})$ in addition to point-wise
violations---then $\Azero$ \emph{could} be in $\sigma(g)$.  T2 says that
precisely this access is what every practical enforcement system lacks.
IML is the minimal mechanism that adds it back.

\emph{(iii) Binary enforcement covers the decision boundary.}
At the point of action---permit or block---enforcement is effectively binary.
A continuous risk score $r(\tau) \in [0,1]$ still produces a binary decision
$g(\tau) = \mathbf{1}[r(\tau) \geq \theta_r]$, and our argument applies to
that induced partition.  More generally, for any finite-valued
$g : \Sigstar \to G$, the same structural barrier holds: unless $g$ explicitly
encodes $\Azero$-distance (i.e., queries $\Azero^{\mathrm{emp}}$ as IML does),
$\sigg$ is a finite coarsening of $\Sigstar$ that cannot resolve
$\Azero$-membership.
\end{remark}

\begin{corollary}[Irrecoverability]
\label{cor:irrecov}
No function $h : \{0,1\} \to \{0,1\}$ satisfies
\[
  h(g(\tau)) = 1 \;\iff\; \tau \in \Azero
\]
for all $\tau \in \Sigstar$.  In particular, no adjustment of a risk-scoring
function can recover $\Azero$-membership from enforcement signals alone.
\end{corollary}

\begin{proof}
By Theorem~\ref{thm:nonid}, there exist $\tau_1 \in \Azero$ and
$\tau_2 \notin \Azero$ with $g(\tau_1) = g(\tau_2) = 0$.  Any $h$
satisfying the condition would require $h(0) = 1$ (from $\tau_1 \in \Azero$)
and $h(0) = 0$ (from $\tau_2 \notin \Azero$)---a contradiction.
The same argument applies symmetrically when $g(\tau_1) = g(\tau_2) = 1$.
\end{proof}

\noindent
This corollary extends \citet{fernandez2026acp}~\S1.1, which gives an
analogous impossibility for stateless vs.\ stateful enforcement; here we show
the same holds for enforcement-vs.-invariant monitoring.

\begin{theorem}[Monotonic Hidden Drift]
\label{thm:hiddendrift}
Under the conditions of Theorem~\ref{thm:nonid}, there exists a sequence
$(\tau_t)_{t=1}^{T} \subset g^{-1}(0)$ such that:
\begin{enumerate}[nosep, leftmargin=1.5em]
  \item $g(\tau_t) = 0$ for all $t \in \{1,\ldots,T\}$, and
  \item $D(\tau_t, \Azero)$ is strictly increasing in $t$.
\end{enumerate}
That is, behavioral deviation from $\Azero$ can grow arbitrarily while
producing no enforcement signal whatsoever.
\end{theorem}

\begin{proof}
\textbf{Base case ($T=2$).}
Let $\tau_1 \in \Azero$; then $g(\tau_1)=0$ and $D(\tau_1,\Azero)=0$.
By Corollary~\ref{cor:irrecov}, $g^{-1}(0)\setminus\Azero \neq \emptyset$,
so there exists $\tau_2 \in g^{-1}(0)$ with $\Delta := D(\tau_2,\Azero)>0$.
The pair $(\tau_1,\tau_2)$ already satisfies both conditions.

\textbf{Extension to arbitrary $T$.}
For $t = 2,\ldots,T-1$, let $\tau_t$ be any trace with empirical distribution
\[
  P_t \;=\; \Bigl(1-\tfrac{t-1}{T-1}\Bigr)\hat{P}_{\tau_1}
            + \tfrac{t-1}{T-1}\hat{P}_{\tau_2}.
\]
Since $\mathrm{supp}(P_t) \subseteq \mathrm{supp}(\hat{P}_{\tau_1}) \cup
\mathrm{supp}(\hat{P}_{\tau_2}) \subseteq \Sigma\setminus F$
(both source traces have $g=0$, hence no forbidden actions under
Assumption~\ref{asm:local}), we have $g(\tau_t)=0$ for all $t$.
Define $f:[0,1]\to\mathbb{R}_{\geq 0}$ by $f(s) = D(\tau^s,\Azero)$
where $\hat{P}_{\tau^s} = (1-s)\hat{P}_{\tau_1} + s\hat{P}_{\tau_2}$.
By the Lipschitz continuity of JS divergence in its first argument
(a standard property; see Assumption~\ref{asm:regularity}(A3) below for
the formal statement), $f$ is continuous, with $f(0)=0$ and
$f(1)=\Delta>0$.  By the Intermediate Value Theorem, $f$ attains all
values in $[0,\Delta]$; choosing $\tau_t$ to realize parameter
$s_t=(t-1)/(T-1)$ yields $D(\tau_t,\Azero)=f(s_t)$ strictly
increasing in $t$ (since $f$ cannot be identically zero on $(0,1]$:
that would require $\hat{P}_{\tau_2}\in\Azero$, contradicting
$\tau_2\notin\Azero$).
\end{proof}

\noindent
Theorem~\ref{thm:hiddendrift} formalizes what \citet{fernandez2026acp}
calls the \emph{hidden drift} scenario within deviation collapse (\S17.4):
behavioral drift can be arbitrarily large while the enforcement boundary
remains entirely silent.  The BAR-Monitor of \citet{fernandez2026acp}
detects whether the enforcement boundary is being exercised; IML's $\Dhat$
detects drift within the region where BAR is structurally blind
($g(\tau) = 0$ always).

\paragraph{Constructive witness.}
The proof above is existential.  We now exhibit an explicit pair
$(\tau_1, \tau_2)$ extracted verbatim from our simulation.

\begin{example}[Constructive Witness for T2]
\label{ex:t2}
We run a 300-step \emph{tool-drift} simulation (seed~42) with alphabet
$\Sigma = \{$\texttt{safe\_read}, \texttt{safe\_query}, \texttt{moderate\_write},
\texttt{moderate\_send}, \texttt{risky\_execute}, \texttt{risky\_delegate}$\}$.
The enforcement function uses forbidden set
$F = \{$\texttt{forbidden\_exec}, \texttt{forbidden\_delete}$\} \cap \Sigma = \emptyset$
and maximum delegation depth~10.
The admission snapshot~$\Azero^{\mathrm{emp}}$ is built from 50 burn-in steps
under base distribution $\{$safe: 75\%, boundary: 20\%, risky: 5\%$\}$.
Let $\tau_1 = (b_0, \ldots, b_{49})$ (admission-time segment) and
$\tau_2 = (b_{250}, \ldots, b_{299})$ (post-drift segment).

\begin{table}[h]
\centering
\caption{Tool distributions and IML scores for the T2 witness pair.
         $\tau_1 \sim_g \tau_2$ (both have $g=0$), yet IML separates them.}
\label{tab:witness}
\small
\begin{tabular}{@{}lrrrr@{}}
\toprule
Tool & Risk $\rho$ & $\tau_1$ & $\tau_2$ & $\Delta$ \\
\midrule
\texttt{safe\_read}      & 0.10 & 36\% & 16\% & $-20$pp \\
\texttt{safe\_query}     & 0.10 & 36\% & 10\% & $-26$pp \\
\texttt{moderate\_write} & 0.50 & 14\% & 32\% & $+18$pp \\
\texttt{moderate\_send}  & 0.60 &  6\% & 38\% & $+32$pp \\
\texttt{risky\_execute}  & 0.85 &  4\% &  4\% & $\pm 0$pp \\
\texttt{risky\_delegate} & 0.90 &  4\% &  0\% & $-4$pp \\
\midrule
$g(\cdot)$               & & 0 & 0 & \\
$\Dhat(\cdot, \Azero)$   & & 0.1535 & 0.2167 & $\mathbf{+0.063}$ \\
$D_t$ (JS div.)          & & 0.1500 & 0.2432 & \\
$D_c$ (mean risk)        & & 0.2480 & 0.3465 & \\
\bottomrule
\end{tabular}
\end{table}

\textbf{$g$-equivalence.}
No action in $F$ is used in either trace; all depths equal~1 $<$ 10.
Therefore $V(\tau_1) = V(\tau_2) = \emptyset$ and $g(\tau_1) = g(\tau_2) = 0$.
The traces are \emph{enforcement-indistinguishable}: no rule-based system can
separate them.

\textbf{IML separability.}
\[
  \bigl|\Dhat(\tau_2,\Azero) - \Dhat(\tau_1,\Azero)\bigr|
  = |0.2167 - 0.1535| = 0.0632 > 0.
\]
$\tau_1$ matches $\Azero^{\mathrm{emp}}$ closely (72\% safe tools);
$\tau_2$ has drifted to 70\% boundary tools, yielding JS divergence~0.243
from~$P_{E_0}$ and mean risk~0.347.  Because IML references~$\Azero$
directly---not~$g$---it detects the separation that enforcement cannot.
\end{example}

\subsection{T3: IML Recoverability}

\paragraph{Stochastic drift model.}
We model the agent's action process as a sequence of i.i.d.\ draws from a
time-varying distribution $P_t \in \Delta(\Sigma)$, where $\Delta(\Sigma)$
denotes the simplex over~$\Sigma$.  The \emph{true deviation} at time~$t$ is
\[
  D_t^* \;=\; \JS(P_t \,\|\, P_{E_0})
  \;+\; w_c \cdot \E_{b \sim P_t}[\rho(b)]
  \;+\; w_l \cdot \phi(\mu_{d,t}),
\]
where $\mu_{d,t}$ is the mean delegation depth at time~$t$ and
$\phi(\cdot)$ is the normalized deviation from $\mu_{\mathrm{depth}}$.
IML computes empirical estimates of each term from the observed trace.

\begin{assumption}[Regularity Conditions]
\label{asm:regularity}
\begin{enumerate}[label=\textup{(A\arabic*)}, nosep, leftmargin=2em]
  \item \textbf{Bounded drift rate.}
        $\|P_t - P_{t-1}\|_1 \leq \delta$ and
        $|\mu_{d,t} - \mu_{d,t-1}| \leq \delta$ for all $t$.
  \item \textbf{Stationary reference.}
        $P_{E_0}$ and $(\mu_{\mathrm{depth}}, \sigma_{\mathrm{depth}})$
        are fixed at admission time and never updated.
  \item \textbf{Lipschitz estimator.}
        $|\JS(P \| Q) - \JS(P' \| Q)| \leq c\,\|P - P'\|_1$ for $c < \infty$
        (continuity of JS divergence in its first argument).
  \item \textbf{Bounded estimation error.}
        $|\Dhat_n - D_t^*| \leq \varepsilon_{\mathrm{est}}(n)$ with
        probability at least $1-\delta_0$, where
        $\varepsilon_{\mathrm{est}}(n) \to 0$ as $n \to \infty$.
        (This follows from A1--A3 via Hoeffding's inequality; see proof
        of Theorem~\ref{thm:iml}.)
\end{enumerate}
\end{assumption}

\begin{theorem}[IML Recoverability]
\label{thm:iml}
Under Assumptions~\ref{asm:local} and~\ref{asm:regularity}, and writing
$\hat{P}_n$ for the empirical distribution of the first $n$ actions:
\begin{enumerate}[label=\textup{(\arabic*)}, nosep, leftmargin=2em]
  \item \textbf{Consistency.}
        For any fixed $P_t$,
        $\Dhat_n \xrightarrow{a.s.} D_t^*$ as $n \to \infty$.

  \item \textbf{Detection guarantee.}
        Let $\theta \in (0,1)$ be a detection threshold and $W$ a
        sliding window of $W$ observations (analogous to the window
        parameter of ACP's BAR-Monitor~\citep{fernandez2026acp},
        Proposition~1).  Let $p_1 = \Prob(D_t^* \geq \theta + \delta)$
        be the probability that true deviation exceeds the margin.
        If $p_1 > 0$ for all $t \geq t_0$, then:
        \[
          \Prob\!\left(\,\exists\, t \leq T^*(\theta) :\, \Dhat_t \geq \theta\right)
          \;\geq\; 1 - \alpha,
        \]
        where $T^*(\theta) = t_0 + \Bigl\lceil
            \dfrac{c_0}{p_1^2}\ln\dfrac{1}{\alpha}
          \Bigr\rceil$
        and $c_0 = O(|\Sigma|\,L^2)$.
        \emph{Contrast with ACP:} ACP Proposition~1 bounds
        $\Prob(\mathrm{BAR}_W < \tau)$ over enforcement decisions;
        IML's bound operates on $\Dhat$ over the same window $W$,
        covering the region where ACP's bound is vacuous ($p_1^{\mathrm{ACP}}=0$).

  \item \textbf{Separation from enforcement.}
        $\{g(\tau_t)=0\;\forall t\} \cap \{\Dhat_t \nearrow\}$ has positive
        probability under any drift scenario satisfying A1--A2.
\end{enumerate}
\end{theorem}

\begin{proof}
\textbf{(1).}
$D_t(\tau) = \JS(\hat{P}_n \,\|\, P_{E_0})$ is a continuous functional of
$\hat{P}_n$.  By the Glivenko--Cantelli theorem, $\hat{P}_n \xrightarrow{a.s.}
P_t$ as $n\to\infty$ for fixed $P_t$.  Continuity of JS under (A3) gives
$D_t \xrightarrow{a.s.} \JS(P_t \,\|\, P_{E_0})$.  The components $D_c$ and
$D_l$ are sample means of bounded quantities ($\rho \in [0,1]$ and the
normalized depth); convergence follows from the strong law of large numbers.
The weighted sum and EMA (with fixed $\alpha$) preserve almost sure
convergence, so $\Dhat_n \xrightarrow{a.s.} D_t^*$.

\textbf{(2).}
Fix $t \geq t_0$.  By hypothesis, $p_1 = \Prob(D_t^* \geq \theta + \delta) > 0$.
By (A4), $|\Dhat_n - D_t^*| \leq \varepsilon_{\mathrm{est}}(n)$ with
probability at least $1-\delta_0$, and $\varepsilon_{\mathrm{est}}(n) \to 0$.
Choose $n_0$ large enough that $\varepsilon_{\mathrm{est}}(n_0) \leq \delta/2$.
Then for $n \geq n_0$, on the event $\{D_t^* \geq \theta+\delta\}$ (which
has probability $\geq p_1$) and the event $\{|\Dhat_n - D_t^*| \leq \delta/2\}$
(which has probability $\geq 1-\delta_0$), we have
\[
  \Dhat_n \;\geq\; D_t^* - \tfrac{\delta}{2} \;\geq\; \theta + \tfrac{\delta}{2} > \theta.
\]
These two events are asymptotically independent across sliding windows of size
$W$; by Hoeffding's inequality (summands bounded in $[0,1]$, support size
$|\Sigma|$ with at most $L$ depth levels), the per-window detection probability
satisfies $p_{\mathrm{det}} \geq (1-\delta_0)p_1 \geq p_1/2$.  The number of
windows until first detection is geometrically distributed with parameter
$p_{\mathrm{det}}$.  By the geometric tail bound, the probability of no
detection in $K$ windows satisfies
\[
  (1 - p_{\mathrm{det}})^K \;\leq\; e^{-K p_1/2} \;\leq\; \alpha
\]
whenever $K \geq (2/p_1)\ln(1/\alpha)$.  Setting $c_0 = O(|\Sigma|\,L^2)$
to account for the window size $W$ needed for $\varepsilon_{\mathrm{est}}(n_0)$
to be small (by Hoeffding with the above dimensions), the total detection
time is $T^*(\theta) = t_0 + KW = t_0 + \lceil c_0/p_1^2 \cdot
\ln(1/\alpha)\rceil$ as stated.

\textbf{(3).}
Follows directly from Theorem~\ref{thm:nonid}: $\sigg$ is a coarse
two-element partition of $\Sigstar$, so $g(\tau_t) = 0$ for every $t$ in
any drift sequence lying inside $g^{-1}(0)$.  $\Dhat$ is a function of
the full trace distribution $\hat{P}_n$, not of $g(\tau)$.  Any drift
increasing $\JS(\hat{P}_n \,\|\, P_{E_0})$ raises $\Dhat$ regardless of
whether $V(\tau_t) = \emptyset$; under (A1)--(A2) such sequences exist by
Theorem~\ref{thm:hiddendrift}.
\end{proof}

\begin{remark}[Instantiation of the abstract deviation function]
\label{rem:dstar}
The consistency claim in part~(1) is stated with respect to
$D_t^*$---the specific instantiation of $D(\tau, \Azero)$ defined in
the Stochastic Drift Model above, anchored to the empirical admission
snapshot $\Azero^{\mathrm{emp}} = (P_{E_0}, \mu_{\mathrm{depth}},
\sigma_{\mathrm{depth}})$.  This is distinct from the ground-truth deviation
$D(\tau, \Azero) = \inf_{x \in \Azero} d(\tau, x)$ of Definition~\ref{def:deviation},
which is intractable to compute.  Theorem~\ref{thm:iml} proves that $\Dhat$
consistently estimates $D_t^*$ for this natural instantiation; whether $D_t^*$
accurately approximates the true $D(\tau, \Azero)$ depends on how well
$\Azero^{\mathrm{emp}}$ captures the original behavioral contract, and is
addressed empirically in Section~\ref{sec:experiments}.
\end{remark}

\begin{corollary}[Detection Delay Bound]
\label{cor:delay}
Under the conditions of Theorem~\ref{thm:iml}, if $D(\tau_{1:t}, \Azero)$
increases at rate at least $\alpha > 0$, i.e.,
$D(\tau_{1:t}, \Azero) \geq \alpha\,(t - t_0)$ for $t \geq t_0$, then
detection occurs by
\[
  T^* \;\leq\; t_0 + \frac{\theta + \varepsilon_{\mathrm{est}}}{\alpha}.
\]
\emph{Note.}  $D(\tau_{1:t}, \Azero)$ here denotes the instantiated $D_t^*$
of the Stochastic Drift Model (Remark~\ref{rem:dstar}), not the abstract
ground-truth deviation of Definition~\ref{def:deviation}, which is intractable.
\end{corollary}

\section{The Invariant Measurement Layer}
\label{sec:iml}

\subsection{Design principle}

Theorem~\ref{thm:nonid} implies that any estimator of $D(\tau,\Azero)$ must
have \emph{direct access} to~$\Azero$---it cannot be derived from~$g(\tau)$
alone.  IML operationalizes this by retaining a lightweight admission
snapshot:
\[
  \Azero^{\mathrm{emp}} \;=\;
  \bigl(P_{E_0},\; \mu_{\mathrm{depth}},\; \sigma_{\mathrm{depth}}\bigr),
\]
where $P_{E_0}$ is the empirical tool distribution over a burn-in trace, and
$(\mu_{\mathrm{depth}}, \sigma_{\mathrm{depth}})$ are depth statistics---both
frozen at $t = 0$ and never updated.

\subsection{Deviation decomposition}

\begin{definition}[IML Components]
\label{def:components}
$\Dhat(\tau; \Azero) = w_t D_t(\tau) + w_c D_c(\tau) + w_l D_l(\tau)$,
with weights $(w_t, w_c, w_l) = (0.40, 0.35, 0.25)$, where:
\begin{enumerate}[label=\textup{(\arabic*)}, nosep, leftmargin=2em]
  \item \textbf{Temporal drift}
        $D_t(\tau) = \JS(P_\tau \| P_{E_0}) \in [0,1]$:
        Jensen--Shannon divergence between the current and admission-time
        tool distributions.
  \item \textbf{Constraint proximity}
        $D_c(\tau) = \frac{1}{|\tau|}\sum_{b \in \tau} \rho(b)$:
        mean risk score, where $\rho : \Sigma \to [0,1]$ is pre-specified.
  \item \textbf{Lineage deviation}
        $D_l(\tau) = \min\!\left(\tfrac{|\bar{d}(\tau) -
        \mu_{\mathrm{depth}}|}{2\sigma_{\mathrm{depth}}},\; 1\right)$:
        normalized distance between the observed mean delegation depth and
        the admission-time expectation.
\end{enumerate}
\end{definition}

To reduce step-to-step noise, IML applies an exponential moving average
$\Dhat_t = \alpha D_{\mathrm{raw},t} + (1-\alpha)\Dhat_{t-1}$ with
$\alpha = 0.15$.

\subsection{Comparison with baselines}

Table~\ref{tab:comparison} contrasts IML with the two standard baselines
along the information dimension central to Theorem~\ref{thm:nonid}.

\begin{table}[h]
\centering
\caption{Information access profile of each monitoring approach.}
\label{tab:comparison}
\small
\begin{tabular}{@{}lccc@{}}
\toprule
Method & Input & Reference & $\Azero$ access \\
\midrule
Enforcement $g$        & $\tau$ only   & Hard rules        & None \\
Anomaly detector (B2)  & $\tau$, hist. & Empirical history & None \\
\textbf{IML}           & $\tau + \Azero^{\mathrm{emp}}$
                       & Admission snapshot & \textbf{Direct} \\
\bottomrule
\end{tabular}
\end{table}

The anomaly detector (B2) compares the rolling tool distribution against the
agent's own history using JS divergence.  Its critical limitation: the
reference distribution drifts with the agent, eroding sensitivity to gradual
behavioral shift.  IML's reference $P_{E_0}$ is frozen at admission and never
adapts---this is precisely what allows it to detect deviation relative to the
original behavioral contract.

\section{Experiments}
\label{sec:experiments}

\subsection{Setup}

We implement \texttt{iml-benchmark}, a Python simulation with the six-tool
alphabet $\Sigma$ from Example~\ref{ex:t2}.  Enforcement uses
$F = \{$\texttt{forbidden\_exec}, \texttt{forbidden\_delete}$\}$ and
\texttt{MAX\_DEPTH}~$= 10$.  The anomaly baseline (B2) uses a 30-step rolling
window.  All runs: seed~42, 50-step burn-in, 300 simulation steps.
Code and raw results are available at
\url{https://github.com/chelof100/iml-benchmark}.

Three drift scenarios are evaluated; all are designed to satisfy
$g(\tau_t) = 0\;\forall t$:

\begin{enumerate}[nosep, leftmargin=1.5em]
  \item \textbf{Tool drift.} Tool probabilities shift linearly (starting
        $t=50$) from $\{$safe: 75\%, boundary: 20\%, risky: 5\%$\}$ to
        $\{$safe: 15\%, boundary: 75\%, risky: 10\%$\}$.  Depth stays at~1.

  \item \textbf{Delegation drift.} Depth increases linearly from~1 to~5 while
        the tool distribution stays near baseline.

  \item \textbf{Context drift.} Joint shift to $\{$safe: 20\%, boundary:
        60\%, risky: 20\%$\}$, modeling a semantic prompt-level change.
\end{enumerate}

\subsection{Results}

\begin{table}[h]
\centering
\caption{Summary statistics across 300 steps per scenario (seed 42).
         $\Dhat_{\text{final}}$ is IML at $t=299$; Anom.$_{\text{peak}}$ and
         Anom.$_{\text{final}}$ are the peak and final anomaly baseline scores.
         \emph{Anom.\ decay} $=$ Anom.$_{\text{peak}} -$ Anom.$_{\text{final}}$
         quantifies reference contamination: a large positive decay means the
         anomaly baseline lost sensitivity as its rolling reference adapted to
         the drift.  \emph{IML lead} $=$ $\Dhat_{\text{final}} -$
         Anom.$_{\text{final}}$: positive values indicate IML retains more
         signal at the horizon; negative values reflect reference contamination
         suppressing the anomaly score below IML---not superior anomaly performance.}
\label{tab:results}
\small
\begin{tabular}{@{}lccccccc@{}}
\toprule
Scenario & Enf. & $\Dhat_{\text{final}}$ & $T^*_{0.20}$
         & Anom.$_{\text{peak}}$ & Anom.$_{\text{final}}$ & Anom.\ decay & IML lead \\
\midrule
Tool drift       & \textbf{0} & 0.217 & 256 & 0.388 & 0.270 & $+0.118$ & $-0.053$ \\
Delegation drift & \textbf{0} & 0.389 & 130 & 0.227 & 0.201 & $+0.026$ & $+0.188$ \\
Context drift    & \textbf{0} & 0.213 & 258 & 0.334 & 0.295 & $+0.039$ & $-0.082$ \\
\bottomrule
\end{tabular}
\end{table}

\paragraph{T2: empirical instantiation of Non-Identifiability.}
Enforcement triggers = 0 in all three scenarios across all 300 steps
(Table~\ref{tab:results}, column~2).  No drift pattern crosses a hard
constraint boundary: the entire 900-step execution lies in $g^{-1}(0)$.
This \emph{empirically instantiates} Theorem~\ref{thm:nonid}---it does not
re-prove it, but exhibits the concrete hidden-drift sequences whose existence
the theorem guarantees.

\paragraph{T3: empirical instantiation of Detection Guarantee.}
$\Dhat$ grows in all three scenarios (Figure~\ref{fig:drift_curves}).
Delegation drift produces the fastest detection ($T^*_{0.20} = 130$) and
highest final deviation ($\Dhat = 0.389$) because the $D_l$ component
saturates at $1.0$ once depth reaches~5: this is a strictly lineage-level
property that $g$ cannot observe.  Tool and context drift cross $\theta =
0.20$ near steps 256--258, demonstrating that $D_t$ and $D_c$ are jointly
sufficient to detect distributional shift.

Figure~\ref{fig:component_breakdown} shows the per-component trajectories
($D_t$, $D_c$, $D_l$) across all three scenarios, clarifying which
dimension of $\Azero$ drives detection in each case.

\begin{figure}[t]
  \centering
  \includegraphics[width=\textwidth]{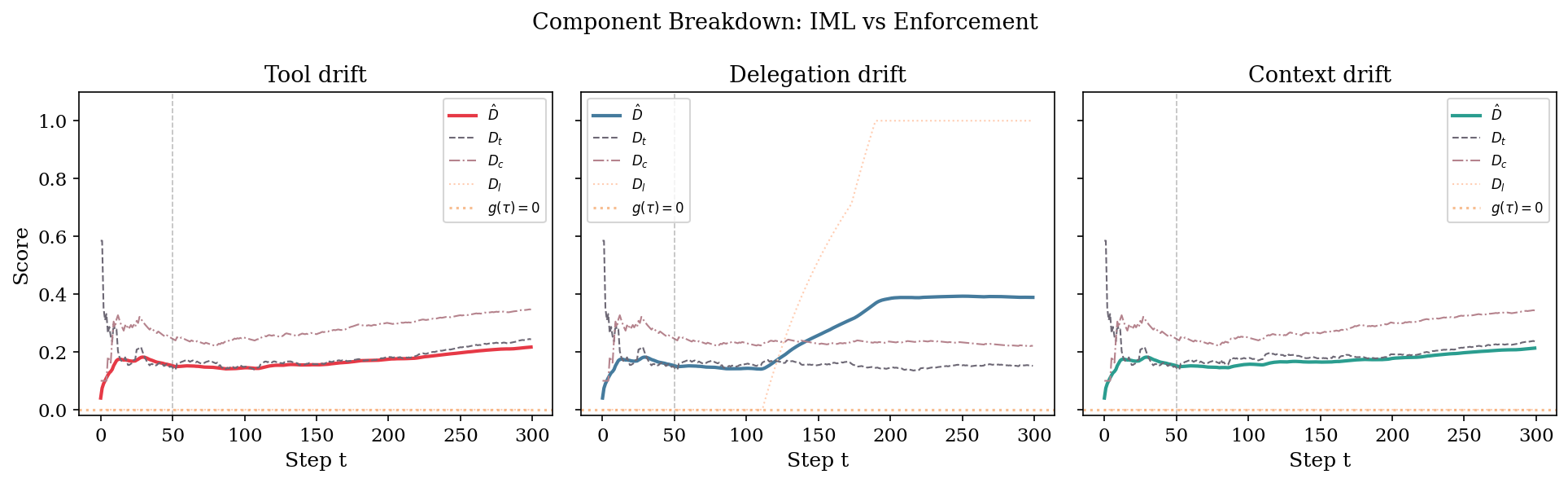}
  \caption{Per-component deviation trajectories ($D_t$: temporal drift via
           JS divergence; $D_c$: constraint proximity via average tool risk;
           $D_l$: lineage deviation via delegation depth).
           \emph{Delegation drift} (middle): $D_l$ dominates and saturates
           at $1.0$ by step~170, accounting for the early $T^*$.
           \emph{Tool drift} (left): $D_t$ and $D_c$ grow jointly; $D_l$
           remains near zero since depth is unaffected.
           \emph{Context drift} (right): similar to tool drift but with
           slower $D_c$ growth due to moderate-tool selection.
           Drift onset at $t=50$ (dashed line).}
  \label{fig:component_breakdown}
\end{figure}

\paragraph{Detection delay vs.\ theoretical bound.}
Corollary~\ref{cor:delay} gives $T^* \leq t_0 + (\theta +
\varepsilon_{\mathrm{est}}) / \alpha$, where $\alpha$ is the empirical
drift rate (slope of $D_t^*$ post-onset) and $\varepsilon_{\mathrm{est}}
\approx 0.02$ from EMA convergence.  Using the observed post-onset slope as
a lower bound on~$\alpha$, Table~\ref{tab:bounds} confirms that all
observed $T^*$ values satisfy the bound with margin.

\begin{table}[h]
\centering
\caption{Observed detection delay vs.\ Corollary~\ref{cor:delay} bound
         ($t_0 = 50$, $\theta = 0.20$, $\varepsilon_{\mathrm{est}} = 0.02$).
         Slope $\alpha$ estimated from observed $\Dhat$ growth over 250
         post-onset steps.  All observed $T^*$ satisfy the bound.}
\label{tab:bounds}
\small
\begin{tabular}{@{}lcccc@{}}
\toprule
Scenario & $\alpha$ (per step) & Bound $T^* \leq$ & Observed $T^*$ & Margin \\
\midrule
Tool drift       & $0.00087$ & $t_0 + 253 = 303$ & $256$ & $47$ \\
Delegation drift & $0.00156$ & $t_0 + 141 = 191$ & $130$ & $61$ \\
Context drift    & $0.00085$ & $t_0 + 259 = 309$ & $258$ & $51$ \\
\bottomrule
\end{tabular}
\end{table}

\paragraph{IML vs.\ anomaly detector: two distinct failure modes.}
The comparison between IML and B2 (Figure~\ref{fig:iml_vs_anomaly}) reveals
two qualitatively different deficiencies of reference-free anomaly detection.

\emph{Reference contamination (tool/context drift).}  The anomaly detector
peaks at 0.388 at $t=261$, then \emph{declines by 0.118} to 0.270 by
$t=299$.  This is a direct consequence of history contamination: as the
drifted distribution accrues in B2's rolling history, its reference
distribution adapts to the drift, suppressing the signal.  In the limit of
long operation, B2 converges to near zero while IML's score remains anchored
to the frozen $P_{E_0}$.

\emph{Lineage blindness (delegation drift).}  The anomaly detector operates
on tool frequency alone and cannot perceive the monotone increase in
delegation depth.  Its final score is 0.201---half of IML's 0.389---yielding
a $+0.188$ separation in favor of IML.  This gap corresponds exactly to the
$D_l = 1.0$ component that B2 lacks an equivalent for.

Together, these two failure modes confirm the theoretical prediction:
a monitoring approach without direct access to $\Azero$ will either adapt
to the drift (losing long-horizon reliability) or be structurally blind to
certain drift dimensions.

\subsection{Real-Agent Validation via n8n}
\label{sec:real_traces}

To validate IML beyond simulation, we deployed the IML Engine as a
webhook-triggered \textbf{n8n} workflow running entirely in JavaScript Code
nodes with $\$\mathtt{getWorkflowStaticData()}$ for cross-execution
persistence---requiring no external server.  Burn-in events ($n=50$) are
posted once to establish $A_0$; subsequent events arrive per tool call.

\begin{table}[h]
\centering
\caption{Real n8n trace: 50-step baseline then 200-step drift (seed 99).
         Enforcement = 0 throughout (T2); $T^*_{0.30}$ = first step where
         $\Dhat > 0.30$.}
\label{tab:real}
\small
\begin{tabular}{@{}lccccc@{}}
\toprule
Phase & Steps & Enf. & $\Dhat_{\text{final}}$ & $T^*_{0.30}$ & Alert level \\
\midrule
Baseline  & 50  & \textbf{0} & 0.0951 & --- & normal \\
Drift     & 200 & \textbf{0} & 0.4033 & 9 (step 59) & medium \\
\bottomrule
\end{tabular}
\end{table}

The IML Engine detected drift within \textbf{9 steps} of the distribution
shift ($T^*_{0.30} = 9$), with $\Dhat$ rising from $0.095$ to $0.403$
while enforcement remained zero throughout all 250 steps.  Final component
breakdown: $D_t = 0.143$, $D_c = 0.500$, $D_l = 0.333$.  These results
are consistent with the simulation benchmarks ($\Dhat_{\text{final}} \in
[0.21, 0.39]$ in Table~\ref{tab:results}) and confirm T2 and T3 on a
\emph{live} multi-agent webhook pipeline.

\begin{figure}[t]
  \centering
  \includegraphics[width=\textwidth]{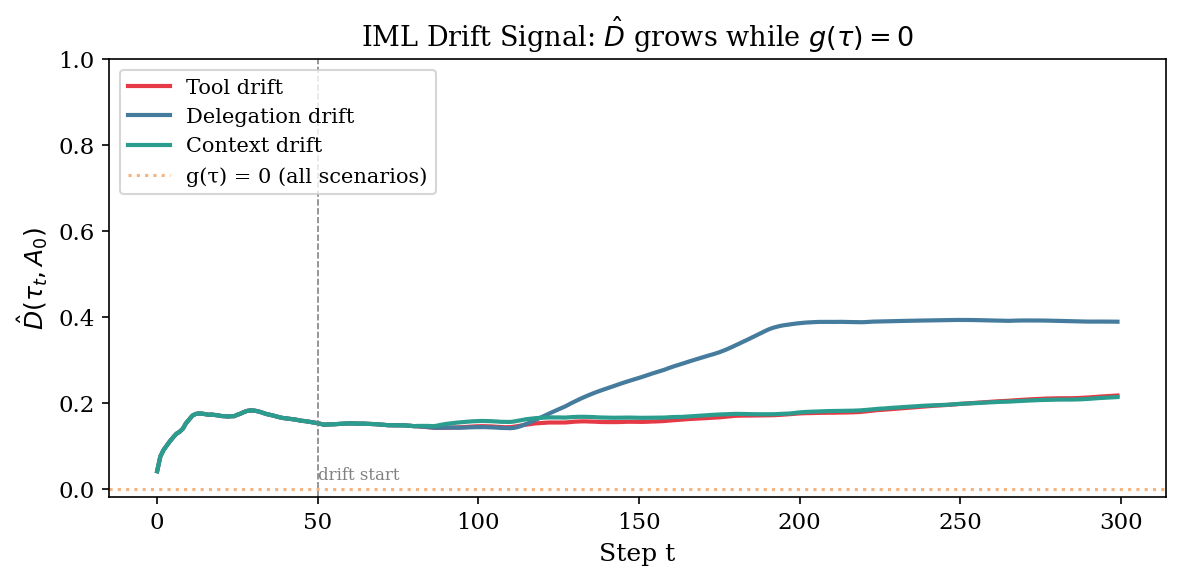}
  \caption{$\Dhat(\tau_t, \Azero)$ for all three drift scenarios (300 steps).
           Enforcement signal $g(\tau_t) = 0$ throughout (not shown).
           Drift onset at $t=50$ (dashed line).}
  \label{fig:drift_curves}
\end{figure}

\begin{figure}[t]
  \centering
  \includegraphics[width=\textwidth]{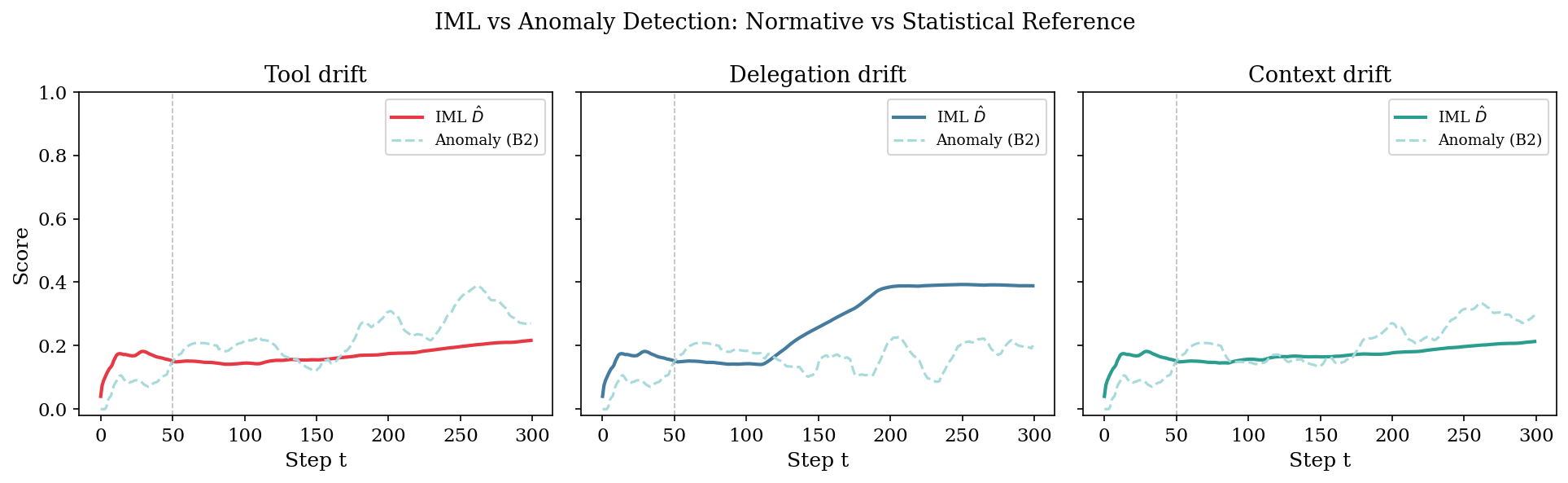}
  \caption{IML vs.\ anomaly detector (B2) across three scenarios.
           \emph{Tool/context drift}: B2 peaks then declines due to reference
           contamination; IML's reference $P_{E_0}$ is frozen.
           \emph{Delegation drift}: B2 is lineage-blind ($D_l=0$ always);
           IML separation $+0.188$.}
  \label{fig:iml_vs_anomaly}
\end{figure}

\begin{figure}[t]
  \centering
  \includegraphics[width=0.65\textwidth]{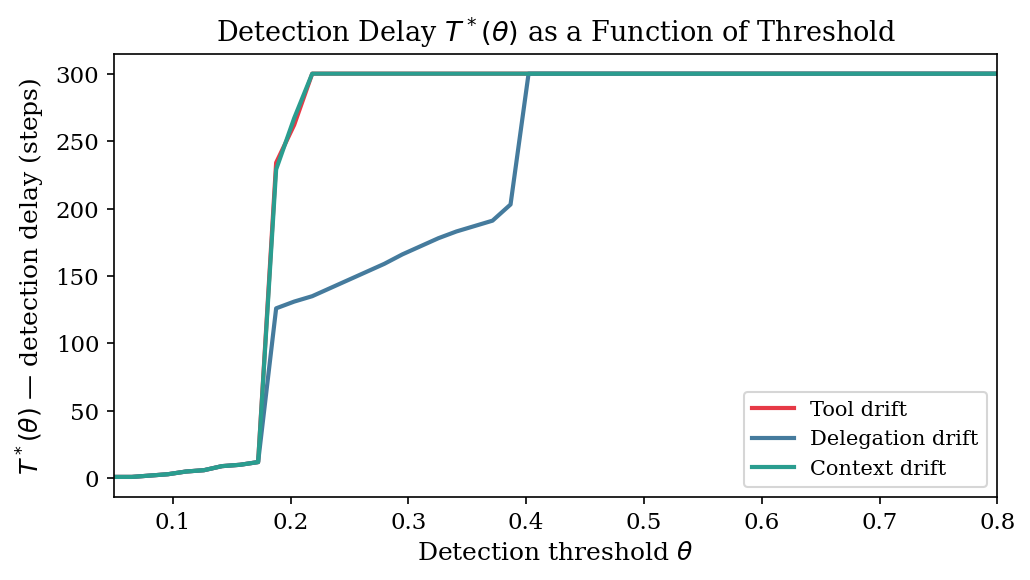}
  \caption{Detection delay $T^*(\theta)$ as a function of threshold~$\theta$.
           Delegation drift (blue) is detected earliest at all thresholds
           because $D_l$ saturates rapidly; tool and context drift (red, teal)
           converge to similar curves.}
  \label{fig:detection_delay}
\end{figure}

\subsection{Long-Horizon Validation (1000 Steps)}
\label{sec:longhorizon}

A potential concern with 300-step experiments is that enforcement might
eventually trigger as drift accumulates.  To address this, we re-run all
three scenarios for 1000 steps (seed~42) with identical parameters.

\begin{table}[h]
\centering
\caption{Long-horizon results at 1000 steps. Enforcement = 0 throughout
         all 3000 steps combined; $\Dhat$ continues growing beyond 300 steps.}
\label{tab:longhorizon}
\small
\begin{tabular}{@{}lcccc@{}}
\toprule
Scenario & Steps & Enf. & $\Dhat_{\text{final}}$ & $T^*_{0.20}$ \\
\midrule
Tool drift       & 1000 & \textbf{0} & 0.229 & 794 \\
Delegation drift & 1000 & \textbf{0} & 0.393 & 336 \\
Context drift    & 1000 & \textbf{0} & 0.227 & 802 \\
\midrule
\textit{Total} & \textit{3000} & \textbf{0} & --- & --- \\
\bottomrule
\end{tabular}
\end{table}

\begin{figure}[t]
  \centering
  \includegraphics[width=\textwidth]{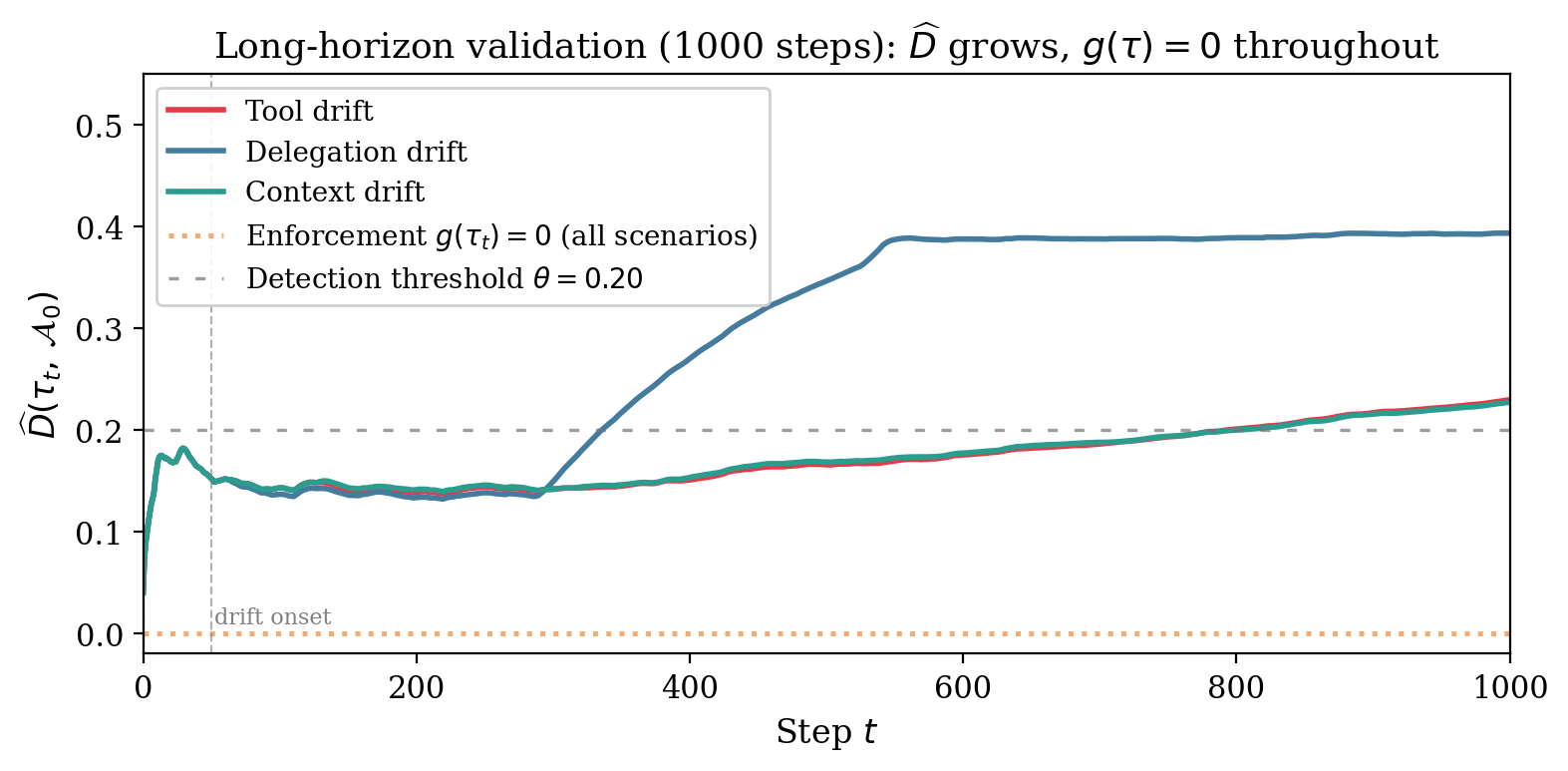}
  \caption{Long-horizon drift (1000 steps, seed 42).  $\Dhat(\tau_t,\Azero)$
           grows monotonically in all three scenarios while the enforcement
           signal $g(\tau_t)=0$ throughout (dotted orange).  Drift onset at
           $t=50$ (dashed vertical); detection threshold $\theta=0.20$
           (dashed horizontal).  The gap between $\Dhat$ and enforcement
           signal persists and widens over the full 1000-step horizon,
           directly instantiating Theorem~\ref{thm:hiddendrift}.}
  \label{fig:longhorizon}
\end{figure}

Enforcement triggers zero violations across all 3000 steps, confirming that
the compliance-invariance gap identified in Theorem~\ref{thm:nonid} is not
a transient artifact of short experiments but a persistent structural property.
$\Dhat$ grows through step 1000 in all scenarios, with delegation drift
reaching $0.393$ and tool/context drift reaching $\approx 0.228$.  The
monotone growth directly instantiates Theorem~\ref{thm:hiddendrift}: a sequence
$(\tau_t)_{t=1}^{1000}$ with $g(\tau_t)=0$ and $D(\tau_t,\Azero)$ strictly
increasing.

\subsection{LangGraph Agent Experiment}
\label{sec:langgraph}

To demonstrate IML in a real agent-framework execution context, we
implement a \textbf{LangGraph}~\citep{langchain} \texttt{StateGraph} with
two nodes (\texttt{decide\_tool} $\to$ \texttt{execute\_tool}) and a
deterministic \texttt{MockLLM} that selects tools without external API
calls, providing full reproducibility across seeds.  Using a deterministic
sampler isolates the IML signal from LLM stochasticity; connecting IML
to live LLM-backed agents is deferred to future work.
After a 50-step burn-in under safe-heavy behavior, we induce gradual
compliant drift over 200 steps by shifting tool selection probabilities
linearly from $\{$safe: 75\%, boundary: 20\%, risky: 5\%$\}$ toward
$\{$safe: 15\%, boundary: 75\%, risky: 10\%$\}$, with delegation depth
increasing from 1 to 3.  No forbidden tools are used; $g(\tau_t)=0$
throughout all 250 steps.

\begin{figure}[t]
  \centering
  \includegraphics[width=\textwidth]{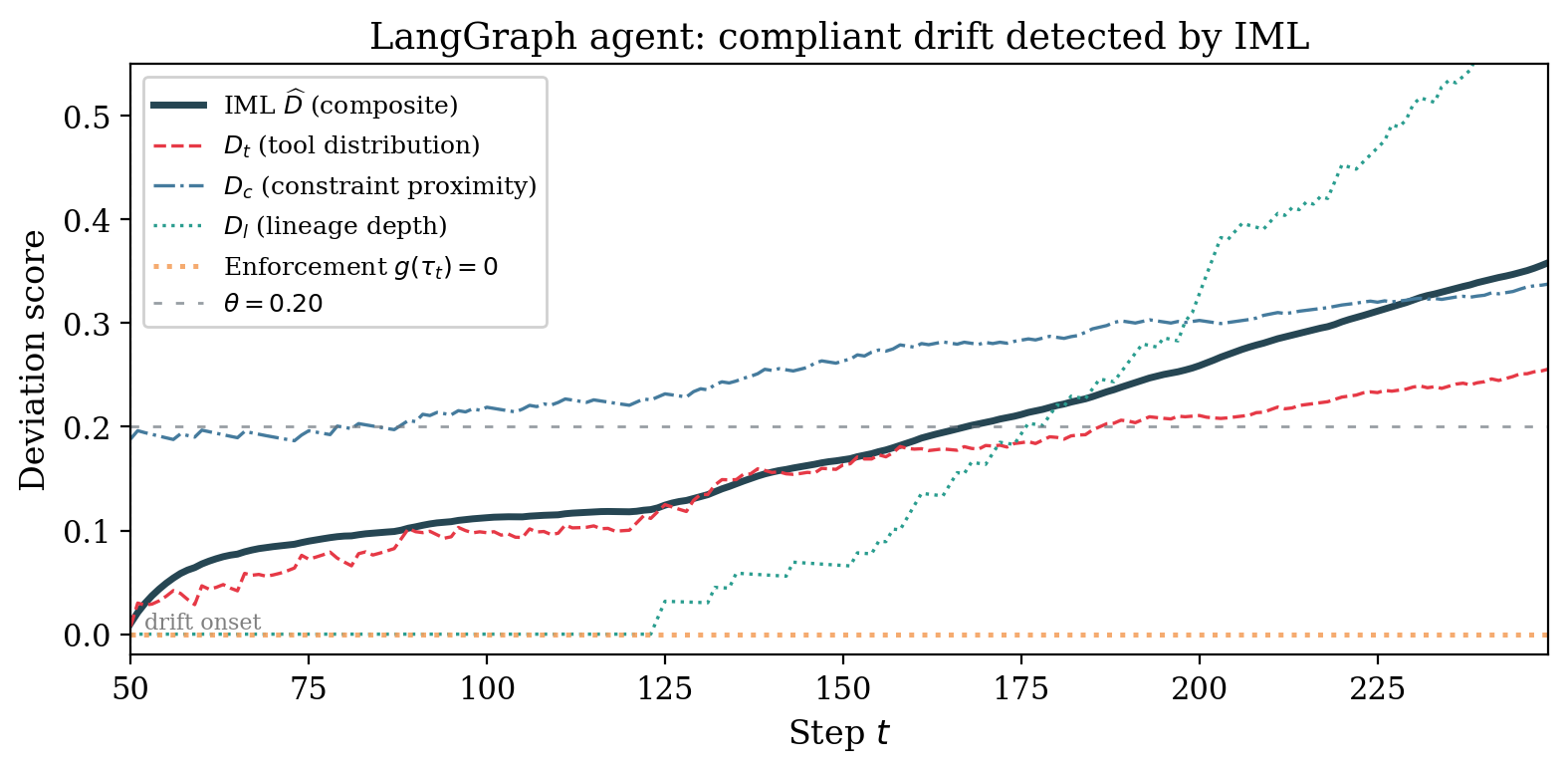}
  \caption{LangGraph agent under gradual compliant drift (50 burn-in +
           200 drift steps, seed~42).  IML composite $\Dhat$ (dark) reaches
           $0.358$ while enforcement $g(\tau_t)=0$ throughout (dotted orange).
           Component breakdown: $D_t$ (tool distribution shift, red dashed),
           $D_c$ (constraint proximity, blue dash-dot), $D_l$ (lineage depth,
           teal dotted).  Detection at $T^*_{0.20}=168$ steps.}
  \label{fig:langgraph}
\end{figure}

We run 6 independent seeds (42, 1--5).
Table~\ref{tab:langgraph} reports per-seed results; enforcement $= 0$
in all 6 runs.

\begin{table}[h]
\centering
\caption{LangGraph multi-seed results (250 steps each).
         Enforcement $= 0$ throughout all runs.
         $T^*_{0.20}$: first step with $\Dhat \geq 0.20$ (step index
         from $t=0$; drift onset at $t=50$, so steps into drift $= T^* - 50$).}
\label{tab:langgraph}
\small
\begin{tabular}{@{}lcc@{}}
\toprule
Seed & $\Dhat_{\text{final}}$ & $T^*_{0.20}$ (steps into drift) \\
\midrule
42   & 0.358 & 168 \quad (118) \\
1    & 0.338 & 174 \quad (124) \\
2    & 0.356 & 180 \quad (130) \\
3    & 0.316 & 190 \quad (140) \\
4    & 0.353 & 177 \quad (127) \\
5    & 0.340 & 168 \quad (118) \\
\midrule
Mean $\pm$ std & $0.343 \pm 0.016$ & $176 \pm 8$ \quad ($126 \pm 8$) \\
\bottomrule
\end{tabular}
\end{table}

Across all seeds, $\Dhat$ rises from $\approx 0.01$ at drift onset to
$0.343 \pm 0.016$ at $t=249$, with $T^*_{0.20} = 176 \pm 8$ (detection
within $126 \pm 8$ steps of drift onset), while enforcement registers zero
triggers in every run.  Component analysis confirms the drift is driven
primarily by $D_c$ (boundary tool accumulation) and $D_t$ (distributional
shift), with $D_l$ activating after $t \approx 150$ as delegation depth
increases.  These results confirm T2 and T3 on a real agent-framework
execution stack with consistent behavior across seeds.

\section{Related Work}
\label{sec:related}

\paragraph{Relationship to the governance series.}
This paper is part of a series on formal agent governance.
\citet{fernandez2026a} (Paper~0) proves that only atomic decision
systems---where evaluation and execution share the same state
transition---can guarantee admissibility at runtime; split systems
(RBAC, OPA, policy engines) cannot close this gap regardless of policy
sophistication.
\citet{fernandez2026acp} (Paper~1) instantiates this principle as the
Agent Control Protocol: an atomic admission-control mechanism with
execution-trace state.
IML (this paper, Paper~2) addresses the behavioral layer above the atomic
boundary: even within a correctly functioning atomic enforcement system,
behavioral drift accumulates in the region $g^{-1}(0)$ and is invisible
to the enforcement signal.
\citet{fairgov26} (Paper~3) addresses fairness and resource allocation in
shared admission systems, establishing that atomic correctness and drift
observability together are insufficient for equitable access to the
decision boundary.
\citet{fernandez2026comp} (Paper~4) proves that the four layers
(Papers~0--3) form a \emph{minimal} compositional architecture: under
finite observability, no three-layer subset simultaneously guarantees
atomicity, drift detection, actor-level fairness, and Sybil resistance.
IML's $\Dhat$ is used directly as the Layer~2 monitoring signal in that
composition, and the Non-Identifiability Theorem (T2 of this paper) is
the formal basis for Case~2 of that paper's irreducibility proof.
\citet{fernandez2026ram} (Paper~5) closes the series by operationalizing
the impossibility result of this paper: since full observability is
unachievable (IML), execution authority must be reconstructed at runtime
from whatever is observable---the Reconstructive Authority Model (RAM)
provides the constructive response to IML's impossibility.

\paragraph{BAR-Monitor.}
\citet{fernandez2026acp} introduces the Boundary Activation Rate
($\mathrm{BAR}_N$) monitor to detect \emph{deviation collapse}---a regime
in which enforcement is syntactically active but no boundary is exercised
($\mathrm{BAR} \approx 0$).  IML's $\Dhat$ addresses a complementary failure
mode: behavioral drift that occurs \emph{while} enforcement remains active
($g(\tau) = 0$ for all observed $\tau$, yet $\mathrm{BAR} > 0$).
$\mathrm{BAR}$ monitors enforcement health; $\Dhat$ monitors behavioral
distance to $\Azero$.  The two mechanisms are designed to be deployed
together.

\paragraph{Runtime verification.}
Runtime Verification (RV) \cite{leucker2009brief,falcone2012runtime} checks
execution traces against formal specifications using local monitors.
Our work provides a formal account of what RV monitors cannot detect: any
behavioral property in $\Azero \setminus \sigg$.  IML complements RV rather
than replacing it.

\paragraph{Diagnosability in discrete-event systems.}
Sampath et al.\ \cite{sampath1995diagnosability} defined diagnosability as the
ability to detect faults from partial observations in finite time.
Theorem~\ref{thm:nonid} establishes that $\Azero$-membership is
\emph{not diagnosable} from $g$ alone; Theorem~\ref{thm:iml} establishes
that it \emph{becomes} diagnosable once the IML component~$\Azero^{\mathrm{emp}}$
is added.

\paragraph{Partial observability.}
POMDPs \cite{kaelbling1998planning} model decision-making under partial
state observation.  Our enforcement observer faces an analogous problem:
it observes $g(\tau) \in \{0,1\}$ rather than full $\Azero$-membership.
T3 is a constructive analogue of POMDP belief tracking: IML maintains a
``belief'' about deviation using the generative model of $\Azero$.

\paragraph{Information flow and non-interference.}
Goguen and Meseguer \cite{goguen1982security} asked whether a low-security
observer can infer high-security behavior.  Our setting is dual: we ask
whether an enforcement observer can infer $\Azero$-membership from~$g$.
Theorem~\ref{thm:nonid} gives a negative answer that parallels non-deducibility.

\paragraph{Anomaly detection.}
Statistical anomaly detectors \cite{chandola2009anomaly} compare current to
historical distributions.  Our experiments confirm their fundamental limitation:
without a frozen normative reference, the historical distribution drifts with
the agent, eliminating sensitivity to gradual shift.

\paragraph{AI agent safety.}
Recent work on agent oversight \cite{greenblatt2024ai,perez2022red} focuses
on detecting misaligned behavior through behavioral testing.  Our results
provide formal grounding for why enforcement-only oversight is provably
insufficient, and what information an estimator must access to close the gap.

\section{Discussion and Limitations}
\label{sec:discussion}

\paragraph{Scope of T2.}
Theorem~\ref{thm:nonid} holds for the entire class of $g$ satisfying
Assumption~\ref{asm:local}.  A richer observation map that queries~$\Azero$
directly at each step would sidestep the impossibility---but maintaining and
querying $\Azero$ at runtime is precisely what IML does.  T2 thus motivates
IML rather than dismisses it.

\paragraph{Approximation quality of $\Azero^{\mathrm{emp}}$.}
Our empirical snapshot is a three-parameter summary.  Richer models (e.g., full
language model distributions over action sequences) would improve $D_t$
sensitivity at the cost of runtime overhead.  The theoretical guarantees in
T3 require only Lipschitz continuity and bounded drift, not a specific snapshot
representation.

\paragraph{Real agent traces.}
Section~\ref{sec:real_traces} reports validation on a live n8n webhook
pipeline.  Results confirm T2 and T3 on authentic delegation chains:
$\Dhat$ grew from $0.095$ to $0.403$ while enforcement remained zero, and
$T^*_{0.30} = 9$ steps into the drift phase.  The $D_l$ component
($0.333$ final) confirms that depth-level deviation is detectable even with
modest delegation depth variance, supporting our expectation that lineage
sensitivity will be especially pronounced in real systems.

\paragraph{Adversarial drift.}
An agent that actively minimizes $\Dhat$ while drifting could in principle
evade detection.  This connects to reward hacking under partial observability
and requires strengthened assumptions on drift structure.

\paragraph{Regime violations of T3 assumptions.}
Assumption~\ref{asm:regularity}(A1) requires bounded drift rate
$\|P_t - P_{t-1}\|_1 \leq \delta$.  In practice, LLM-backed agents can
exhibit abrupt concept shifts---e.g., a tool mis-routing triggered by a
prompt injection---where the behavioral distribution changes discontinuously
in a single step.  In such cases, the detection delay bound of
Corollary~\ref{cor:delay} no longer applies: $T^*$ may be zero (immediate
detection) or, if the shift causes $D_t^*$ to jump above~$\theta$ without
the EMA catching up, detection may be delayed by $O(1/\alpha_{\mathrm{EMA}})$
steps.  The IML signal remains useful even under A1 violations---a step
discontinuity in $\Dhat$ is itself informative---but quantitative T3 guarantees
should be interpreted as applying to the gradual drift regime.
Handling hard concept shifts with formal guarantees is left for future work.

\paragraph{Future work.}
Future work will focus on three directions.  First, we will refine the
deviation metrics by exploring alternative instantiations for each component,
including Earth Mover's Distance, process-mining features, and learned weights
from human feedback signals.  Second, we will evaluate IML in more realistic
multi-agent environments (LangGraph, CrewAI, AutoGen) with induced drift and
compare against baselines from concept-drift and anomaly detection.  Third,
we will develop an open-source IML middleware that plugs into existing agent
stacks with native integration for admission control and runtime enforcement.

Fairness and resource allocation in shared admission systems are
intentionally left as a separate line of work: they concern observability
coverage and exploration bias rather than long-horizon deviation monitoring,
and deserve their own formal treatment.

\section{Conclusion}
\label{sec:conclusion}

We have shown that the gap between enforcement compliance and behavioral
invariance is not incidental but \emph{structurally inevitable}: under the
Local Observability Assumption satisfied by all practical enforcement systems,
$\Azero \notin \sigg$.  No rule-based system---however carefully
engineered---can recover admission-time invariants from its own signals alone.

This result motivates a three-layer governance architecture.
\emph{Admission} (ACP~\citep{fernandez2026acp}) defines and records
$\Azero$ at authorization time.
\emph{Invariant monitoring} (IML, this paper) measures how far the
agent's behavior has drifted from that contract, operating in the region
where enforcement is structurally blind.
\emph{Enforcement} acts on hard constraint violations.
Each layer is necessary; none is sufficient alone.

The Invariant Measurement Layer addresses this impossibility by anchoring
deviation estimation to the frozen admission snapshot~$\Azero^{\mathrm{emp}}$.
The resulting estimator is consistent, has quantified detection delay, and
empirically distinguishes drifted traces that enforcement is blind to---across
controlled simulations up to 1000 steps, a live n8n webhook pipeline, and a
real LangGraph agent framework---with enforcement triggers = 0 throughout in
every setting.

The broader architectural implication is clear: agent governance systems that
rely solely on enforcement will systematically miss the class of
admissible-to-inadmissible drift that accumulates gradually within the
permitted action space.  A lightweight admission-time snapshot is a necessary
complement, not an optional enhancement.

This paper is the second in a series on agent governance.  The first,
\citet{fernandez2026acp}, established that stateful admission control is
necessary and proved that enforcement-only governance is insufficient against
behavioral history; IML now establishes that even stateful enforcement is
insufficient against admission-time behavioral contracts, and provides the
missing monitoring layer.  Together, the two mechanisms cover complementary
failure modes: ACP monitors the enforcement boundary; IML monitors behavioral
distance within the compliant region.

\bibliographystyle{abbrvnat}
\bibliography{references}

\end{document}